\documentclass[11pt]{article}

\usepackage[letterpaper,margin=1in]{geometry}
\usepackage{amsmath,amssymb,amsthm}
\usepackage{newtxtext,newtxmath}
\usepackage{microtype}
\usepackage{booktabs,tabularx,array,multirow,longtable}
\usepackage{graphicx}
\usepackage{tikz}
\usetikzlibrary{arrows.meta,positioning,calc}
\usepackage[section]{placeins}
\usepackage{enumitem}
\usepackage[dvipsnames]{xcolor}
\usepackage{tcolorbox}
\usepackage{listings}
\usepackage[numbers,sort&compress]{natbib}
\usepackage[colorlinks=true,allcolors=MidnightBlue]{hyperref}
\hypersetup{pdftitle={Toward a Science of Intent: Closure Gaps and Delegation Envelopes for Open-World AI Agents}, pdfauthor={Maximiliano Armesto and Christophe Kolb}, pdfsubject={Agentic AI, authorization, intent compilation}, pdfkeywords={agentic AI, intent representation, AI governance, runtime verification, policy-as-code, authorization}}
\usepackage[nameinlink,noabbrev]{cleveref}
\usepackage{lineno}

\setlist[itemize]{leftmargin=*,topsep=2pt,itemsep=2pt}
\setlist[enumerate]{leftmargin=*,topsep=2pt,itemsep=2pt}

\newtheorem{definition}{Definition}
\newtheorem{proposition}{Proposition}
\newtheorem{hypothesis}{Hypothesis}

\newcommand{\sem}{\mathrm{sem}}
\newcommand{\evi}{\mathrm{evid}}
\newcommand{\proc}{\mathrm{proc}}
\newcommand{\inst}{\mathrm{inst}}

\newcommand{\compile}{\mathrm{compile}}
\newcommand{\search}{\mathrm{search}}
\newcommand{\escalate}{\mathrm{escalate}}
\newcommand{\execute}{\mathrm{execute}}
\newcommand{\wait}{\mathrm{wait}}

\lstdefinestyle{yamlstyle}{
  basicstyle=\ttfamily\small,
  breaklines=true,
  frame=single,
  columns=fullflexible,
  keepspaces=true,
  showstringspaces=false,
  backgroundcolor=\color{gray!5},
  xleftmargin=1em,
  xrightmargin=1em
}

\tcbset{
  colback=gray!4,
  colframe=gray!50,
  boxrule=0.4pt,
  arc=2pt,
  left=6pt,right=6pt,top=5pt,bottom=5pt
}

\title{\vspace{-1.5em}\textbf{Toward a Science of Intent: Closure Gaps and Delegation Envelopes for Open-World AI Agents}}
\author{Maximiliano Armesto \quad Christophe Kolb\\
Taller Technologies\\[-0.2em]
{\small\texttt{maximiliano.armesto@tallertechnologies.com}}\\[-0.2em]
{\small\texttt{christophe.kolb@tallertechnologies.com}}}
\date{}

\begin{document}
\maketitle

\begin{abstract}
Recent work has framed intelligence in verifiable tasks as reducing time-to-solution through learned structure and test-time search, while systems work has explored learned runtimes in which computation, memory and I/O migrate into model state. These perspectives do not explain why capable models remain difficult to deploy in open institutions. We propose \emph{intent compilation}: the transformation of partially specified human purpose into inspectable artifacts that bind execution. The relevant deployment distinction is closed-world solver versus open-world agent. In closed worlds, a checker is largely given; in open worlds, verification is distributed across semantic, evidentiary, procedural and institutional dimensions. We formalize this residual openness as a closure-gap vector, define delegation envelopes as pre-authorized regions of action space, distinguish misclosure from undersearch, and outline benchmark metrics for testing when closure interventions outperform additional inference-time search.
\end{abstract}


\section{Introduction}

Recent work on agents emphasizes time: once a task is verifiable, learned structure and test-time compute reduce time-to-solution \citep{simon1955behavioral,zilberstein1996anytime,lewis2014computational,achille2025agents,snell2024scaling}. In parallel, work on world models and learned runtimes asks how much computation, memory and I/O can migrate into learned state \citep{ha2018world,zhuge2026neuralcomputers,hafner2025mastering}. These programs explain faster solving, richer planning and reusable computation. They do not settle a different question: when is a candidate output authorized to become an institutional action?

A large class of deployed-system failures arises not from inability to generate a plausible answer, but from failure to bind it to the right task, admissible evidence, permitted procedure or legitimate authority. A legal assistant may draft a correct clause but not be authorized to send it; a security agent may identify a compromised host but lack authority to delete data; a travel agent may find a cheaper itinerary using a stale policy document. The failure is not merely weak reasoning. It is missing closure.

This paper names the missing object: \emph{intent compilation}, the externalization of partially specified purpose into artifacts that bind a stochastic runtime: what is asked, what evidence may be used, what procedure may be followed and who is authorized to act. The compiler analogy is disciplined. A classical compiler lowers explicit source code into executable form; intent compilation starts earlier, from goals, assumptions, risk tolerances and institutional roles, and targets a composite runtime of model, evidence substrate, tool harness, monitors and governance shell.

The language-model-versus-world-model distinction matters for internal mechanism, but it is insufficient for deployment. The relevant distinction is \emph{closed-world solver versus open-world agent}. In a closed-world task, the checker is largely given: legal moves, success criteria and admissible evidence are settled in advance. In an open-world task, the checker is distributed across unresolved semantic, evidentiary, procedural and institutional conditions. Additional search may help when those conditions are settled; otherwise longer reasoning traces can become fluent work over an undercompiled problem.

\paragraph{Contributions.} The paper defines intent compilation, closure gaps, time-to-authorized-action, delegation envelopes and misclosure, then proposes benchmark metrics for testing when closure interventions outperform additional search.

\paragraph{Methodological status and scope.} This article is a conceptual framework and research agenda, not an empirical report of a new deployed system. It defines intent compilation and constructs for testing it: closure gaps, contracts, delegation envelopes, misclosure and time-to-authorized-action. Claims about compilation versus search are hypotheses and benchmark designs. Author-developed systems named later are illustrative, not independent evidence.

\section{Preliminaries: task episodes, actions, and openness}

We use ``open-world'' operationally rather than in the narrow knowledge-representation sense. A task episode is open when the conditions for valid action remain materially underspecified at run time. This underspecification may concern meaning, evidence, method, or authority.

\begin{definition}[Task episode]
A task episode is a tuple
\[
    e=(u,x,A,R,P,h_0),
\]
where $u$ is a user request, $x$ is non-linguistic context, $A$ is the available action space, $R$ is the evidence substrate, $P$ is the operative policy or institutional rule set, and $h_0$ is the initial interaction history. A task episode is \emph{closed} to the extent that the validity conditions for actions in $A$ are already specified. It is \emph{open} to the extent that those conditions must be resolved during the episode.
\end{definition}

\begin{definition}[Action]
An action is not merely text. We write
\[
    a=(op,obj,content,actor,tool,t),
\]
where $op$ is an operation, $obj$ the object of the operation, $content$ the informational payload, $actor$ the acting role or principal, $tool$ the execution channel, and $t$ the time. The same textual content can be authorized as analysis, prohibited as execution, or mandatory to escalate depending on the other components of the action.
\end{definition}

The surface string underdetermines the action. ``Notify the customer'' may mean drafting, sending, logging or requesting approval. ``Fix the vendor issue'' may require retrieving a contract, opening a ticket, issuing a credit or escalating to legal. Linguistic competence is only one part of the episode.

\begin{definition}[Intent compilation]
Given a task episode $e$, intent compilation is the process of producing a contract tuple
\[
    K_t=(S_t,E_t,M_t,I_t),
\]
where $S_t$ specifies task semantics, $E_t$ specifies evidence admissibility, $M_t$ specifies permitted method, and $I_t$ specifies institutional authority. The compiler also exposes residual closure gaps indicating which parts of the episode remain insufficiently specified for autonomous action.
\end{definition}

The notation $M_t$ denotes method and avoids overloading $P$, which denotes operative policy. The contracts are architectural roles, not product categories. They can be implemented by typed schemas, retrieval rules, workflow engines, monitors, access-control systems, approval queues or combinations of these mechanisms.

\section{Closure gaps and time-to-authorized-action}

A closed-world task can be evaluated against a checker that is assumed to exist. An open-world task requires the system to help construct, retrieve, or ratify parts of the checker before acting. We represent the unresolved remainder by a closure-gap vector
\[
    C_t=(C_{\sem,t}, C_{\evi,t}, C_{\proc,t}, C_{\inst,t}).
\]
The four components ask: what exactly is being asked; on what grounds; by what method; and under whose authority?

\begin{definition}[Closure gap]
Let $\mathcal{J}=\{\sem,\evi,\proc,\inst\}$. For $i\in\mathcal{J}$, let $\rho_i(K_t,h_t)$ be an idealized indicator that contract dimension $i$ is sufficiently specified for the class of actions under consideration at history $h_t$. A probabilistic closure gap is
\[
    C_{i,t}=1-\Pr\bigl(\rho_i(K_t,h_t)=1\mid h_t,K_t\bigr).
\]
A cost-sensitive closure gap is
\[
    C^{c}_{i,t}=\mathbb{E}\left[\operatorname{cost\_to\_close}_i\mid h_t,K_t\right].
\]
Observable signals such as clarification frequency, citation conflicts, retry depth, or permission-denied events are proxy features for estimating these latent gaps. They are not the gaps themselves.
\end{definition}

The four-way decomposition is operational, not exhaustive. It mirrors the recurring structure of authorized action: content, grounds, method and authority. Risk is not treated as a fifth contract; it is a cross-cutting severity parameter that changes the threshold of closure required before action. A reversible formatting change may tolerate modest closure; a medical recommendation, production deployment, or legal commitment requires much stricter closure across all contracts.

\subsection{From latency decomposition to event traces}

It is tempting to write
\[
    T_{\mathrm{authorized}}\approx T_{\compile}+T_{\search}+T_{\escalate}.
\]
This expression is rhetorically useful but formally coarse. Compilation, search and escalation may overlap; retrieval can be both evidence compilation and search; clarification can invalidate prior reasoning; and escalation can halt the episode. We therefore define time-to-authorized-action over an event trace.

Let an episode trace be
\[
    \tau = ((q_1,s_1,f_1,c_1),\ldots,(q_n,s_n,f_n,c_n)),
\]
where each event has type $q_j\in\{\compile,\search,\escalate,\execute,\wait\}$, start time $s_j$, finish time $f_j$, and accounting cost $c_j$. The wall-clock time-to-authorized-action is
\[
    T_{\mathrm{authorized}} = \inf\{t: \exists a_t \in A_t \text{ such that } a_t \text{ is ratified or lies inside the delegation envelope}\} - s_0.
\]
The accounting weight for event class $q$ is
\[
    W_q(\tau)=\sum_{j:q_j=q} c_j.
\]
In sequential episodes, $T_{\mathrm{authorized}}$ may approximate the sum of compile, search, escalation, and waiting intervals. In concurrent or preemptive episodes, $W_{\compile}$, $W_{\search}$, and $W_{\escalate}$ should be interpreted as cost categories over the trace rather than as disjoint wall-clock intervals.

\begin{hypothesis}[Closure intervention hypothesis]
On task distributions stratified by high closure gap, interventions that reduce the dominant $C_i$ will improve time-to-authorized-action and contract compliance more cost-effectively than additional inference-time search, holding base model capability constant. On low-closure-gap tasks, additional search should dominate.
\end{hypothesis}

This hypothesis is conditional: the marginal value of search depends on whether the world is sufficiently closed for search to be meaningful.

\subsection{Routing over loci of intelligence}

The closure vector induces a routing policy over loci of intelligence: internal inference, clarification, retrieval, simulation or tool use, escalation, and abstention. Table~\ref{tab:routing} gives the operational version.

\begin{table}[!htbp]
\centering
\small
\begin{tabularx}{\textwidth}{>{\raggedright\arraybackslash}p{0.18\textwidth} >{\raggedright\arraybackslash}p{0.22\textwidth} >{\raggedright\arraybackslash}p{0.20\textwidth} >{\raggedright\arraybackslash}X}
\toprule
Dominant unresolved dimension & Diagnostic signal & Best next move & Closure condition and failure if skipped \\
\midrule
Internal search gap & Clear contract but weak candidate, failing tests, low confidence under agreed checker & Think, plan, search, verify & Candidate meets the checker. If skipped: premature or weak solution. \\
Semantic gap & Ambiguous referents, unstable acceptance criteria, repeated clarification questions & Ask, disambiguate, instantiate task schema & Acceptance criteria are explicit. If skipped: solving the wrong problem. \\
Evidentiary gap & Stale source, missing provenance, conflicting citations, inadmissible document & Retrieve, cite, reconcile, reject inadmissible source & Grounds satisfy evidence policy. If skipped: fluent reasoning from prohibited or untraceable evidence. \\
Procedural gap & Tool mismatch, retry loops, missing rollback, workflow violation & Simulate, use approved tool path, monitor, define rollback & Method satisfies workflow and safety constraints. If skipped: brittle or unsafe actuation. \\
Institutional gap & Role mismatch, missing approval, policy conflict, irreversible action outside role & Escalate, seek approval, abstain & Permission predicate is satisfied or escalation is logged. If skipped: correct answer, wrong decision. \\
\bottomrule
\end{tabularx}
\caption{Routing policy over closure gaps. Additional search is one intervention among several. The correct next move depends on which component of residual openness dominates.}
\label{tab:routing}
\end{table}

\begin{figure}[!htbp]
\centering
\includegraphics[width=0.50\textwidth]{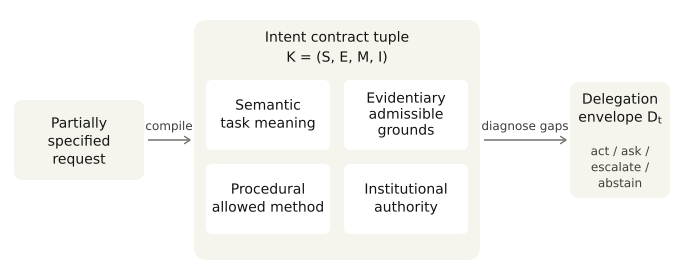}
\caption{Intent compilation as a publication-oriented schematic. A partially specified request is externalized into four contracts that expose residual closure gaps and define a delegation envelope for autonomous action, clarification, escalation or abstention.}
\label{fig:pipeline}
\end{figure}

\section{The four-contract stack}

Intent compilation produces ratification surfaces: points at which a person, team or institution can inspect what will bind the stochastic runtime. The four contracts are separable because each fails differently, and coupled because authorized action requires their conjunction.

\subsection{Semantic contract}

The semantic contract specifies the task ontology: entities, relations, output schema, acceptance criteria, forbidden interpretations, and ambiguity tolerances. Contemporary prompts sometimes play this role, but prompts are not durable contracts. A robust semantic contract should be typed, versioned, inspectable, and linked to acceptance criteria. This connects intent compilation to requirements engineering and problem frames \citep{lamsweerde2009requirements,jackson2001problem}, while shifting the representation target from deterministic software artifacts to stochastic action runtimes.

A semantic failure occurs when the system solves a plausible but unauthorized problem. For example, ``modernize this service'' remains open until boundaries, preserved behavior, release scope, regression tests and acceptance criteria are explicit.

\subsection{Evidentiary contract}

The evidentiary contract specifies admissible sources, freshness requirements, provenance duties, uncertainty representation, and conflict-resolution rules. Retrieval is one mechanism for evidentiary closure, but the deeper issue is admissibility. A system that reasons from a stale, prohibited, or untraceable source is not epistemically aligned merely because its conclusion is plausible. Retrieval-augmented generation, provenance standards, and software signing systems provide components of this layer \citep{lewis2020rag,moreau2015prov,newman2022sigstore}; intent compilation treats admissibility as part of the task contract rather than as an incidental retrieval setting.

\subsection{Procedural contract}

The procedural contract defines allowed tools, search budgets, tests, stop conditions, retry rules, monitoring hooks, rollback paths and escalation triggers. These obligations are often scattered across orchestration code, CI pipelines and convention; they should be part of the task's executable method. Workflow theory, Petri nets, BPMN/YAWL, runtime verification, and agent-computer interfaces provide relevant machinery \citep{murata1989petri,vanderaalst2005yawl,leucker2009runtime,bartocci2018introduction,yang2024sweagent,wang2024openhands}.

A procedural failure occurs when a system uses the wrong method even if the answer is acceptable: a patch from an unlogged tool chain may pass tests but fail review because the method is unauditable.

\subsection{Institutional contract}

The institutional contract defines role, permission, privacy boundary, approval path, logging requirement, rollback right, liability surface, and separation of duties. This layer is technical because it changes the semantics of action. The same content may be permissible as a draft, prohibited as a sent message, or required as an escalation report. Policy-as-code, RBAC, XACML, Cedar, capability systems, separation of duties and accountable algorithms provide ingredients for this layer \citep{sandhu1996role,dennis1966programming,miller2006robust,oasis2013xacml,sandall2021opa,aws2023cedar,kroll2017accountable,hadfield2017rules}.

Throughout, ``authorized'' means authorized under an explicit operative policy. It does not imply that the policy is ethically justified, socially desirable, or complete. Intent compilation makes authorization inspectable; it does not guarantee that the institution's rules are good.

\subsection{Running example: travel rebooking}

Consider the request ``Rebook the trip because the vendor moved the meeting.'' A closed-world solver might search flight options and select an itinerary. An open-world agent must first compile the request.

\begin{tcolorbox}
\textbf{Semantic contract:} identify the traveler, trip, meeting time, cabin constraints, arrival deadline, and acceptance criteria.\par
\textbf{Evidentiary contract:} use the current booking record, corporate travel policy, airline inventory less than 15 minutes old, and approved employee profile; reject stale policy documents.\par
\textbf{Procedural contract:} retrieve current booking, check policy, search alternatives, compare fare delta, hold fare if allowed, draft notification, and release any hold if escalation fails.\par
\textbf{Institutional contract:} act autonomously only for domestic rebooking, same cabin, and fare delta below a declared cap; escalate for international travel, visa risk, policy exception, or fare delta above cap.
\end{tcolorbox}

Inside the resulting envelope, the agent may place a fare hold and draft a notification. At the boundary, it asks about a near-cap fare increase. Outside the envelope, it must not purchase a ticket or alter an international itinerary without approval. The itinerary is only one part of the authorized action.

\section{A typed intent representation}

Free-form prompts are too unstable to carry meaning, evidence policy, procedure and authority. The target is a composable intermediate representation that binds stochastic search without requiring full formalization. Existing stacks provide pieces--JSON Schema for structure, workflow languages for procedure, Rego or Cedar for policy, provenance predicates for evidence--but do not type the four contracts as one object \citep{sandall2021opa,aws2023cedar,vanderaalst2005yawl,moreau2015prov}. The following fragment illustrates the target rather than prescribing a syntax.

\begin{lstlisting}[style=yamlstyle,caption={Illustrative typed intent object for the travel-rebooking example.},label={lst:schema}]
task:
  objective: "rebook employee travel after meeting change"
  action_type: "travel_rebooking"

semantic_contract:
  entities: [traveler_id, booking_id, destination]
  acceptance_criteria:
    - "arrival before meeting_start_time"
    - "same cabin unless approved"
  ambiguity_policy:
    missing_meeting_time: "ask"

evidentiary_contract:
  admissible_sources:
    - "corporate_travel_policy:current_version"
    - "current_booking_record"
    - "airline_inventory_api:<15min"
  conflict_resolution: "policy overrides preference"

procedural_contract:
  workflow:
    - retrieve_booking
    - check_policy
    - search_alternatives
    - hold_fare_if_within_envelope
  rollback: "release_fare_hold"

institutional_contract:
  autonomous_if: "domestic && fare_delta <= 200 && same_cabin"
  escalate_if: "international || fare_delta > 200 || visa_risk"
  audit: "retain_logs_365_days"
\end{lstlisting}

The value is not perfect formalization. It is that unresolved fields become visible rather than hidden inside latent model behavior. Compilation relocates ambiguity upward: the remaining ambiguity becomes an explicit decision about whether to ask, retrieve, simulate, escalate or abstain.

\section{Delegation envelopes}

Ratification need not collapse autonomy into perpetual human approval. Institutions already delegate authority by ratifying rules rather than every individual action. Agentic AI needs the same construct.

\begin{definition}[Deterministic delegation envelope]
For contract tuple $K_t=(S_t,E_t,M_t,I_t)$ and action space $A_t$, the deterministic delegation envelope is
\[
    D_t(K_t)=\{a\in A_t: \sigma_t(a)\wedge \epsilon_t(a)\wedge \pi_t(a)\wedge \iota_t(a)\},
\]
where $\sigma_t$, $\epsilon_t$, $\pi_t$, and $\iota_t$ are the semantic, evidentiary, procedural, and institutional predicates induced by the contracts. Inside $D_t$, autonomous action is authorized under the operative policy; at the boundary, the system must ask; outside it, the system must escalate or abstain.
\end{definition}

Many real envelopes are probabilistic because predicates are inferred from natural language, retrieved evidence, or model judgments. Risk changes the threshold.

\begin{definition}[Risk-sensitive probabilistic envelope]
Let $R(a)$ be an action severity or materiality score, and let $\alpha(R(a))$ be a nondecreasing authorization threshold. The probabilistic delegation envelope is
\[
    D_t^{\alpha}(K_t)=\left\{a\in A_t:
    \Pr_t\bigl[\sigma_t(a)\wedge\epsilon_t(a)\wedge\pi_t(a)\wedge\iota_t(a)\mid h_t,K_t\bigr]
    \ge \alpha(R(a))\right\}.
\]
A higher-severity action requires higher confidence that all four predicates hold.
\end{definition}

Examples are simple: draft but do not send; recommend but do not approve; isolate a host but do not delete data; place a fare hold but do not purchase; modify a branch but do not deploy to production. The point is not to minimize human involvement everywhere. It is to make the boundary of autonomous action explicit, auditable, and reusable.

\subsection{Qualified envelope properties}

The delegation-envelope formalism supports testable properties, but each requires careful qualification.

\begin{proposition}[Monotonic tightening]
Fix an action space $A_t$. If every revised predicate is a logical strengthening of the previous one--$\sigma'_t\Rightarrow\sigma_t$, $\epsilon'_t\Rightarrow\epsilon_t$, $\pi'_t\Rightarrow\pi_t$, and $\iota'_t\Rightarrow\iota_t$--then $D'_t\subseteq D_t$.
\end{proposition}

\begin{proof}[Proof sketch]
If $a\in D'_t$, then $a$ satisfies all revised predicates. By implication, it satisfies all original predicates, so $a\in D_t$. The result does not hold for revisions that add actions, add roles, add tools, or split a coarse action into safer sub-actions; those are envelope revisions, not monotone tightenings.
\end{proof}

\begin{proposition}[Composite actions require global sequence constraints]
For a composite action sequence $\alpha=(a_1,\ldots,a_k)$, component authorization is necessary but not sufficient. The induced sequence envelope is
\[
    D_t^{*}=\{\alpha: (\bigwedge_{j=1}^{k} a_j\in D_{t,j})\wedge G_t(\alpha)\},
\]
where $G_t$ captures ordering constraints, separation of duties, cumulative risk, rollback dependencies and approval inheritance.
\end{proposition}

Individually authorized steps can be jointly prohibited. For example, a user may be allowed to draft a payment instruction and separately approve a payment, while the institution prohibits the same principal from doing both in one workflow. Conversely, a supervisor-approved composite workflow may authorize steps that would not be authorized outside that workflow.

Envelope stability is measurable only after defining a perturbation distribution:
\[
    S_{\Delta}(D)=1-\mathbb{E}_{\delta\sim\Delta}\left[
    \frac{\mu(D(K)\triangle D(K+\delta))}{\mu(D(K)\cup D(K+\delta))}
    \right].
\]
Here $\Delta$ ranges over semantically preserving or minimally contrastive contract perturbations, $\triangle$ is symmetric difference, and $\mu$ measures action space. An envelope that flips under paraphrase has low stability.

\section{Misclosure: failure taxonomy and benchmark protocol}

Answer quality and authorization quality are separable. This separation motivates a new failure class.

\begin{definition}[Misclosure]
A misclosure failure occurs when a system's candidate output would satisfy a plausible task-level competence checker but cannot be ratified, or is wrongly executed, because one or more semantic, evidentiary, procedural, or institutional contracts was absent, underspecified, incorrectly inferred, or violated.
\end{definition}

Misclosure is not the same as getting the answer wrong. It is also not the same as refusing too often. Table~\ref{tab:taxonomy} separates adjacent failures.

\begin{table}[!htbp]
\centering
\small
\begin{tabularx}{\textwidth}{>{\raggedright\arraybackslash}p{0.18\textwidth} >{\raggedright\arraybackslash}p{0.38\textwidth} >{\raggedright\arraybackslash}X}
\toprule
Failure type & Definition & Example \\
\midrule
Undersearch & Contract is clear, but the candidate is wrong or weak. & Code patch fails declared tests despite clear requirements. \\
Misclosure & Candidate may be good, but not ratifiable under one or more contracts. & Correct answer uses a stale policy document or executes without approval. \\
Overclosure & Envelope is too narrow relative to policy and risk. & Agent escalates a trivial reversible action that policy allows autonomously. \\
Misdelegation & Authority predicate is wrong or inferred from the wrong principal. & Agent sends an email when only drafting was authorized. \\
Contract conflict & Two contracts impose inconsistent requirements. & Evidence policy requires a source that the privacy policy forbids accessing. \\
Premature internalization & A recurrent routine is collapsed into opaque behavior before its contract is stable. & Agent silently reuses an outdated approval pattern. \\
\bottomrule
\end{tabularx}
\caption{Misclosure and neighboring failure modes. The taxonomy prevents answer error, authorization error, and over-escalation from being conflated.}
\label{tab:taxonomy}
\end{table}

\subsection{Benchmark construction}

A misclosure benchmark should separate competence from authorization. Each task instance should contain: (i) a candidate answer or action with ordinary task-quality evaluation; (ii) a contract tuple defining semantic, evidentiary, procedural and institutional validity; (iii) perturbations that alter one contract while preserving surface plausibility; and (iv) a ratification oracle that determines whether action is allowed.

Existing agent benchmarks provide useful seeds for task realism and tool interaction \citep{jimenez2024swebench,yao2024taubench,mialon2023gaia,liu2023agentbench,liang2023helm}, but they do not isolate misclosure. Table~\ref{tab:benchmark} gives the perturbation protocol.

\begin{table}[!htbp]
\centering
\small
\begin{tabularx}{\textwidth}{>{\raggedright\arraybackslash}p{0.16\textwidth} >{\raggedright\arraybackslash}p{0.34\textwidth} >{\raggedright\arraybackslash}p{0.26\textwidth} >{\raggedright\arraybackslash}X}
\toprule
Perturbation & Construction & Correct behavior & Failure if missed \\
\midrule
Semantic & Replace the task description with a plausible variant that changes acceptance criteria while preserving surface wording. & Ask, instantiate the schema, or solve under the revised criteria. & Solves the wrong problem. \\
Evidentiary & Age, corrupt, remove, or policy-disallow a relevant-looking source. & Reject inadmissible source, retrieve current evidence, or disclose uncertainty. & Acts on stale or prohibited grounds. \\
Procedural & Disallow the canonical tool path and provide a sanctioned alternative. & Use approved method, simulate first, or escalate if no method exists. & Uses an unsafe or unauditable workflow. \\
Institutional & Assign a role lacking permission for execution while permitting analysis or drafting. & Escalate, seek approval, draft only, or abstain. & Executes a correct answer without authority. \\
\bottomrule
\end{tabularx}
\caption{Minimal misclosure benchmark protocol. Each perturbation targets one contract while preserving enough surface plausibility that additional search alone should not reliably fix the failure.}
\label{tab:benchmark}
\end{table}

\subsection{Experimental design}

A pilot experiment could use three institutional domains: software change requests, travel or procurement workflows, and customer-support or compliance actions. For each episode, construct four contract perturbations. Compare a baseline tool-using model, the same model with more test-time compute, retrieval, typed contracts, typed contracts plus routing policy, and an oracle-contract condition. Primary outcomes are unauthorized-action rate, time-to-authorized-action, contract compliance, escalation precision and recall, provenance completeness, and ratification burden. A mixed-effects model over task, domain, intervention and perturbation type could test whether closure interventions outperform additional search on high-closure-gap tasks.

\begin{hypothesis}[Misclosure separability]
Perturbation benchmarks can induce failures that persist under additional search but disappear under targeted closure interventions. In such cases, the failure is misclosure rather than undersearch.
\end{hypothesis}

\begin{hypothesis}[Envelope prediction]
Deployment success in open institutional tasks is better predicted by delegation-envelope size, stability, and compliance than by answer accuracy or pass@$k$ alone.
\end{hypothesis}

\section{Metrics for authorized action}

Table~\ref{tab:metrics} defines metrics needed to evaluate whether an agent is merely answer-generating or action-authorizing.

\begingroup
\small
\begin{longtable}{>{\raggedright\arraybackslash}p{0.25\textwidth} >{\raggedright\arraybackslash}p{0.68\textwidth}}
\caption{Metrics for authorized action. Answer accuracy remains useful, but it is insufficient for open-world deployment.}\label{tab:metrics}\\
\toprule
Metric & Operational definition \\
\midrule
\endfirsthead
\toprule
Metric & Operational definition \\
\midrule
\endhead
\midrule
\multicolumn{2}{r}{Continued on next page}\\
\endfoot
\bottomrule
\endlastfoot
Time-to-authorized-action & Wall-clock from task declaration to first ratified action or action inside the envelope. \\
Contract-latency breakdown & Accounting weights $W_{\compile}$, $W_{\search}$, $W_{\escalate}$, $W_{\execute}$, and $W_{\wait}$ over the event trace. \\
Ratification burden & Human-seconds per authorized action, summed over all ratification surfaces. \\
Escalation precision & Fraction of escalations in which human review changed the outcome or confirmed a required approval boundary. \\
Escalation recall & Fraction of episodes requiring escalation in which the system escalated before acting. \\
False-autonomy rate & Fraction of unauthorized actions executed without escalation. \\
Over-escalation rate & Fraction of episodes escalated despite being inside the envelope under the ratification oracle. \\
Provenance completeness & Fraction of action-bearing claims with auditable, admissible sources. \\
Contract compliance & Fraction of executed actions whose semantic, evidentiary, procedural, and institutional predicates all hold. \\
Rollback success & Fraction of reversible actions that can be reverted within a declared time bound. \\
Envelope size & Cardinality or measure of $D_t$ relative to the declared action space. \\
Envelope stability & Expected normalized symmetric difference of accepted-action sets under declared perturbation distribution $\Delta$. \\
Human-review disagreement & Inter-rater variance over ratification decisions, reported for benchmark or field evaluations. \\
\end{longtable}
\endgroup

High answer accuracy with poor false autonomy, provenance or compliance remains answer generation rather than authorization. Escalating every episode is safe but unusable. The target is the Pareto frontier between autonomy, latency, burden and compliance.

\section{Controlled internalization and institutional shells}

Intent compilation is not endless scaffolding. As agents and learned runtimes mature, recurrent structure will be internalized into weights, tools, memory or neural runtime primitives \citep{zhuge2026neuralcomputers}. The open-world correction is that internalization should occur only after intent has first been externalized, ratified and monitored.

\begin{definition}[Controlled internalization]
A routine is eligible for controlled internalization only if it occurs above a declared frequency threshold, has a stable contract representation, exceeds a contract-compliance threshold with confidence intervals, passes perturbation tests for envelope stability, is monitorable at runtime, and has rollback or containment for foreseeable failures.
\end{definition}

Mechanistic interpretability may help diagnose learned substructure, but behavioral stability under contract perturbation and monitorability against runtime policy also matter \citep{olah2020zoom,elhage2022toy}. What must be avoided is collapsing governance debt into opaque capability.

Institutional shells should be runtime objects, not adjacent paperwork. Approval flows, audit trails, retention rules, privacy scopes, role bindings and reversible actuation should participate in runtime semantics. Current systems express fragments of this stack--OPA/Rego, Cedar, RBAC, XACML, Sigstore, runtime monitors and proof-carrying protocols \citep{sandall2021opa,aws2023cedar,sandhu1996role,oasis2013xacml,newman2022sigstore,necula1997proof,brundage2020toward}. Three gaps remain: type-level integration with task semantics, proof-carrying action protocols, and runtime attestation that institutions can audit at deployment speed.

\section{Relation to existing programs}

Intent compilation is adjacent to requirements engineering, policy-as-code, runtime verification, program synthesis, HCI, AI governance and preference-based alignment. Its novelty is to treat the four contracts as one control surface traded against inference-time compute.

\begin{table}[!htbp]
\centering
\small
\begin{tabularx}{\textwidth}{>{\raggedright\arraybackslash}p{0.20\textwidth} >{\raggedright\arraybackslash}p{0.25\textwidth} >{\raggedright\arraybackslash}p{0.25\textwidth} >{\raggedright\arraybackslash}X}
\toprule
Program & What it externalizes & What it often leaves implicit & Difference from intent compilation \\
\midrule
Requirements engineering & Goals, constraints, acceptance criteria & Runtime authority, evidence admissibility, stochastic execution & Treats specification as a runtime control variable, not only a design-time artifact. \\
Policy-as-code & Authorization predicates over requests & Task semantics, evidentiary admissibility, procedural validity & Makes policy part of action semantics and integrates it with the other contracts. \\
Runtime verification & Trace properties and monitors & Intent elicitation, evidence policy, authority ratification & Includes pre-action closure and ratification surfaces, not only post hoc monitoring. \\
Program synthesis & Executable artifacts from specifications or demonstrations & Institutional legitimacy and action authority & Targets authorized action, not only executable code. \\
Constitutional AI and RLHF & General behavioral preferences and principles & Task-specific, inspectable authority and evidence contracts & Uses external ratified contracts rather than only learned preference structure. \\
Agent benchmarks & Task success, tool success, answer quality & Admissibility, role permission, approval, provenance & Evaluates misclosure separately from undersearch. \\
\bottomrule
\end{tabularx}
\caption{Intent compilation in relation to adjacent programs. The distinction is a system-level integration and control claim rather than a claim that the component ideas are absent from prior work.}
\label{tab:relations}
\end{table}

The systems literature already contains partial instances. SWE-agent, OpenHands, Copilot Workspace, Aider, Devin, MetaGPT, ReAct, Reflexion, AutoGen and LangChain instantiate parts of this stack \citep{yang2024sweagent,wang2024openhands,github2024copilot,gauthier2024aider,cognition2024devin,hong2024metagpt,yao2023react,shinn2023reflexion,wu2023autogen,chase2022langchain}. Echo and Chiron, developed by the authors' organization, are additional examples; they are included only to illustrate the architecture and are not independent evidence for the framework \citep{armesto2026software,armesto2026scrat}.

\section{Limitations and boundary conditions}

Intent compilation does not eliminate ambiguity; it relocates ambiguity into inspectable artifacts. Several limits follow. Authorization is not ethical legitimacy: harmful policy can still authorize harmful action. The four-contract stack is operational, not exhaustive. Closure diagnostics can be miscalibrated or gamed, so benchmarks must evaluate false autonomy and over-escalation jointly. Some norms are tacit, contested or rapidly changing; the compiler should reveal missing authority rather than formalize it away. Compilation also has costs: excessive contract burden can make a system unusable or shift labor to reviewers. Finally, this article reports no new empirical evaluation. The framework should be judged by whether future benchmarks and field studies show that it predicts and reduces misclosure failures.

\section{Conclusion}

Capability grows when learned systems internalize more computation. Governed deployability grows when institutions externalize more of the intent, evidence, procedure, and authority that bind action. Open-world agents require both. They should search aggressively when the world is sufficiently closed, ask or retrieve when closure is missing, and escalate or abstain when authority runs out.

Intent compilation is the intermediate layer between human purpose and autonomous execution. It turns latent intent into four contracts, exposes closure gaps, defines delegation envelopes and makes misclosure measurable. The next phase of agentic AI should ask not only whether models can think longer, but whether systems can bind thought to inspectable, ratified, reusable and revocable conditions for action.

\section*{Acknowledgements}
We thank Jonathan Rademacher for drawing our attention to his independently developed Standing Algebra $\Sigma^{R}$, a closure-theoretic admissibility and legitimacy-envelope framework, and to related terminology around envelope projection, closure-based admissibility, and interoperable constraint geometry \citep{rademacher2026standingalgebra}. This acknowledgement is included for scholarly completeness.

\section*{Competing interests}
Both authors hold executive roles at Taller Technologies, where Echo and Chiron are developed. This manuscript cites two recent publications by the authors and mentions Echo and Chiron as illustrative architectural examples alongside external systems. It does not present new unpublished empirical findings about either system.

\section*{Author contributions}
Both authors contributed to the conception of the framework, development of the formal vocabulary, and writing and revision of the manuscript. Both authors approved the final text.

\section*{Data availability}
No new data were generated for this Perspective. Empirical claims are either framed as hypotheses, attributed to cited literature, or described as proposed benchmark protocols.

\bibliographystyle{unsrtnat}
\bibliography{toward_science_of_intent_refs}

@article{achille2025agents,
  author = {Achille, Alessandro and Soatto, Stefano},
  title = {AI Agents as Universal Task Solvers: It's All About Time},
  journal = {arXiv preprint arXiv:2510.12066},
  year = {2025},
  doi = {10.48550/arXiv.2510.12066}
}

@article{snell2024scaling,
  author = {Snell, Charlie and Lee, Jaehoon and Xu, Kelvin and Kumar, Aviral},
  title = {Scaling LLM Test-Time Compute Optimally Can Be More Effective than Scaling Model Parameters},
  journal = {arXiv preprint arXiv:2408.03314},
  year = {2024}
}

@article{simon1955behavioral,
  author = {Simon, Herbert A.},
  title = {A Behavioral Model of Rational Choice},
  journal = {Quarterly Journal of Economics},
  volume = {69},
  number = {1},
  pages = {99--118},
  year = {1955}
}

@article{zilberstein1996anytime,
  author = {Zilberstein, Shlomo},
  title = {Using Anytime Algorithms in Intelligent Systems},
  journal = {AI Magazine},
  volume = {17},
  number = {3},
  pages = {73--83},
  year = {1996}
}

@article{lewis2014computational,
  author = {Lewis, Richard L. and Howes, Andrew and Singh, Satinder},
  title = {Computational Rationality: Linking Mechanism and Behavior through Bounded Utility Maximization},
  journal = {Topics in Cognitive Science},
  volume = {6},
  number = {2},
  pages = {279--311},
  year = {2014}
}

@article{zhuge2026neuralcomputers,
  author = {Zhuge, Mingchen and Zhao, Changsheng and Liu, Haozhe and Zhou, Zijian and Liu, Shuming and Wang, Wenyi and Chang, Ernie and Le Lan, Gael and Fei, Junjie and Zhang, Wenxuan and Sun, Yasheng and Cai, Zhipeng and Liu, Zechun and Xiong, Yunyang and Yang, Yining and Tian, Yuandong and Shi, Yangyang and Chandra, Vikas and Schmidhuber, J{\"u}rgen},
  title = {Neural Computers},
  journal = {arXiv preprint arXiv:2604.06425},
  year = {2026}
}

@article{ha2018world,
  author = {Ha, David and Schmidhuber, J{\"u}rgen},
  title = {World Models},
  journal = {arXiv preprint arXiv:1803.10122},
  year = {2018}
}

@article{hafner2025mastering,
  author = {Hafner, Danijar and Pasukonis, Jurgis and Ba, Jimmy and Lillicrap, Timothy},
  title = {Mastering Diverse Control Tasks through World Models},
  journal = {Nature},
  volume = {640},
  pages = {647--653},
  year = {2025},
  doi = {10.1038/s41586-025-08744-2}
}

@techreport{sandall2021opa,
  author = {Sandall, Tim and Hinrichs, Timothy L.},
  title = {Open Policy Agent: Policy-Based Control for Cloud Native Environments},
  institution = {Cloud Native Computing Foundation},
  year = {2021}
}

@techreport{aws2023cedar,
  author = {{Amazon Web Services}},
  title = {Cedar: A New Policy Language},
  institution = {Amazon Web Services},
  year = {2023}
}

@techreport{oasis2013xacml,
  author = {{OASIS}},
  title = {eXtensible Access Control Markup Language (XACML) Version 3.0},
  institution = {OASIS Standard},
  year = {2013}
}

@article{leucker2009runtime,
  author = {Leucker, Martin and Schallhart, Christian},
  title = {A Brief Account of Runtime Verification},
  journal = {Journal of Logic and Algebraic Programming},
  volume = {78},
  number = {5},
  pages = {293--303},
  year = {2009}
}

@incollection{bartocci2018introduction,
  author = {Bartocci, Ezio and Falcone, Yli{\`e}s and Francalanza, Adrian and Reger, Giles},
  title = {Introduction to Runtime Verification},
  booktitle = {Lectures on Runtime Verification},
  pages = {1--33},
  publisher = {Springer},
  year = {2018}
}

@book{lamsweerde2009requirements,
  author = {van Lamsweerde, Axel},
  title = {Requirements Engineering: From System Goals to UML Models to Software Specifications},
  publisher = {Wiley},
  year = {2009}
}

@book{jackson2001problem,
  author = {Jackson, Michael},
  title = {Problem Frames: Analysing and Structuring Software Development Problems},
  publisher = {Addison-Wesley},
  year = {2001}
}

@book{hadfield2017rules,
  author = {Hadfield, Gillian K.},
  title = {Rules for a Flat World: Why Humans Invented Law and How to Reinvent It for a Complex Global Economy},
  publisher = {Oxford University Press},
  year = {2017}
}

@article{yang2024sweagent,
  author = {Yang, John and Jimenez, Carlos E. and Wettig, Alexander and Lieret, Kilian and Yao, Shunyu and Narasimhan, Karthik and Press, Ofir},
  title = {SWE-agent: Agent-Computer Interfaces Enable Automated Software Engineering},
  journal = {arXiv preprint arXiv:2405.15793},
  year = {2024}
}

@article{wang2024openhands,
  author = {Wang, Xingyao and Li, Boxuan and Song, Yufan and Xu, Frank F. and Tang, Xiangru and Zhuge, Mingchen and Pan, Jiayi and Song, Yueqi and Li, Bowen and Singh, Jaskirat and others},
  title = {OpenHands: An Open Platform for AI Software Developers as Generalist Agents},
  journal = {arXiv preprint arXiv:2407.16741},
  year = {2024}
}

@inproceedings{hong2024metagpt,
  author = {Hong, Sirui and Zhuge, Mingchen and Chen, Jonathan and Zheng, Xiawu and Cheng, Yuheng and Wang, Jinlin and Zhang, Ceyao and Wang, Zili and Yau, Steven Ka Shing and Lin, Zijuan and others},
  title = {MetaGPT: Meta Programming for a Multi-Agent Collaborative Framework},
  booktitle = {International Conference on Learning Representations},
  year = {2024}
}

@misc{github2024copilot,
  author = {{GitHub}},
  title = {GitHub Copilot Workspace: AI-Native Developer Environment},
  year = {2024},
  howpublished = {Product announcement}
}

@misc{gauthier2024aider,
  author = {Gauthier, Paul},
  title = {Aider: AI Pair Programming in Your Terminal},
  year = {2024},
  howpublished = {Software documentation}
}

@misc{cognition2024devin,
  author = {{Cognition AI}},
  title = {Introducing Devin, the First AI Software Engineer},
  year = {2024},
  howpublished = {Product announcement}
}

@inproceedings{yao2023react,
  author = {Yao, Shunyu and Zhao, Jeffrey and Yu, Dian and Du, Nan and Shafran, Izhak and Narasimhan, Karthik and Cao, Yuan},
  title = {ReAct: Synergizing Reasoning and Acting in Language Models},
  booktitle = {International Conference on Learning Representations},
  year = {2023}
}

@inproceedings{shinn2023reflexion,
  author = {Shinn, Noah and Cassano, Federico and Berman, Edward and Gopinath, Ashwin and Narasimhan, Karthik and Yao, Shunyu},
  title = {Reflexion: Language Agents with Verbal Reinforcement Learning},
  booktitle = {Advances in Neural Information Processing Systems},
  year = {2023}
}

@article{wu2023autogen,
  author = {Wu, Qingyun and Bansal, Gagan and Zhang, Jieyu and Wu, Yiran and Li, Beibin and Zhu, Erkang and Jiang, Li and Zhang, Xiaoyun and Zhang, Shaokun and Liu, Jiale and others},
  title = {AutoGen: Enabling Next-Gen LLM Applications via Multi-Agent Conversation},
  journal = {arXiv preprint arXiv:2308.08155},
  year = {2023}
}

@misc{chase2022langchain,
  author = {Chase, Harrison},
  title = {LangChain: Building Applications with LLMs through Composability},
  year = {2022},
  howpublished = {Software documentation}
}

@inproceedings{lewis2020rag,
  author = {Lewis, Patrick and Perez, Ethan and Piktus, Aleksandra and Petroni, Fabio and Karpukhin, Vladimir and Goyal, Naman and K{\"u}ttler, Heinrich and Lewis, Mike and Yih, Wen-tau and Rockt{\"a}schel, Tim and Riedel, Sebastian and Kiela, Douwe},
  title = {Retrieval-Augmented Generation for Knowledge-Intensive NLP Tasks},
  booktitle = {Advances in Neural Information Processing Systems},
  year = {2020}
}

@article{moreau2015prov,
  author = {Moreau, Luc and Groth, Paul and Cheney, James and Lebo, Timothy and Miles, Simon},
  title = {The Rationale of PROV},
  journal = {Journal of Web Semantics},
  volume = {35},
  pages = {235--257},
  year = {2015}
}

@inproceedings{newman2022sigstore,
  author = {Newman, Zachary and Meyers, John Speed and Torres-Arias, Santiago},
  title = {Sigstore: Software Signing for Everybody},
  booktitle = {ACM Conference on Computer and Communications Security},
  year = {2022}
}

@article{vanderaalst2005yawl,
  author = {van der Aalst, Wil M. P. and ter Hofstede, Arthur H. M.},
  title = {YAWL: Yet Another Workflow Language},
  journal = {Information Systems},
  volume = {30},
  number = {4},
  pages = {245--275},
  year = {2005}
}

@article{murata1989petri,
  author = {Murata, Tadao},
  title = {Petri Nets: Properties, Analysis and Applications},
  journal = {Proceedings of the IEEE},
  volume = {77},
  number = {4},
  pages = {541--580},
  year = {1989}
}

@article{armesto2026software,
  author = {Armesto, Maximiliano and Kolb, Christophe},
  title = {Orchestrating Human-AI Software Delivery: A Retrospective Longitudinal Field Study of Three Software Modernization Programs},
  journal = {arXiv preprint arXiv:2603.20028},
  year = {2026}
}

@article{armesto2026scrat,
  author = {Armesto, Maximiliano and Kolb, Christophe},
  title = {Coupled Control, Structured Memory, and Verifiable Action in Agentic AI (SCRAT -- Stochastic Control with Retrieval and Auditable Trajectories): A Comparative Perspective from Squirrel Locomotion and Scatter-Hoarding},
  journal = {arXiv preprint arXiv:2604.03201},
  year = {2026}
}

@article{dennis1966programming,
  author = {Dennis, Jack B. and Van Horn, Earl C.},
  title = {Programming Semantics for Multiprogrammed Computations},
  journal = {Communications of the ACM},
  volume = {9},
  number = {3},
  pages = {143--155},
  year = {1966}
}

@phdthesis{miller2006robust,
  author = {Miller, Mark S.},
  title = {Robust Composition: Towards a Unified Approach to Access Control and Concurrency Control},
  school = {Johns Hopkins University},
  year = {2006}
}

@article{kroll2017accountable,
  author = {Kroll, Joshua A. and Huey, Joanna and Barocas, Solon and Felten, Edward W. and Reidenberg, Joel R. and Robinson, David G. and Yu, Harlan},
  title = {Accountable Algorithms},
  journal = {University of Pennsylvania Law Review},
  volume = {165},
  pages = {633--705},
  year = {2017}
}

@article{sandhu1996role,
  author = {Sandhu, Ravi S. and Coyne, Edward J. and Feinstein, Hal L. and Youman, Charles E.},
  title = {Role-Based Access Control Models},
  journal = {IEEE Computer},
  volume = {29},
  number = {2},
  pages = {38--47},
  year = {1996}
}

@article{olah2020zoom,
  author = {Olah, Chris and Cammarata, Nick and Schubert, Ludwig and Goh, Gabriel and Petrov, Michael and Carter, Shan},
  title = {Zoom In: An Introduction to Circuits},
  journal = {Distill},
  year = {2020}
}

@misc{elhage2022toy,
  author = {Elhage, Nelson and Hume, Tristan and Olsson, Catherine and Schiefer, Nicholas and Henighan, Tom and Kravec, Shauna and Hatfield-Dodds, Zac and Lasenby, Robert and Drain, Dawn and Chen, Carol and others},
  title = {Toy Models of Superposition},
  year = {2022},
  howpublished = {Transformer Circuits Thread}
}

@inproceedings{jimenez2024swebench,
  author = {Jimenez, Carlos E. and Yang, John and Wettig, Alexander and Yao, Shunyu and Pei, Kexin and Press, Ofir and Narasimhan, Karthik},
  title = {SWE-bench: Can Language Models Resolve Real-World GitHub Issues?},
  booktitle = {International Conference on Learning Representations},
  year = {2024}
}

@article{yao2024taubench,
  author = {Yao, Shunyu and Shinn, Noah and Razavi, Pedram and Narasimhan, Karthik},
  title = {{$\tau$-bench}: A Benchmark for Tool-Agent-User Interaction in Real-World Domains},
  journal = {arXiv preprint arXiv:2406.12045},
  year = {2024}
}

@article{mialon2023gaia,
  author = {Mialon, Gr{\'e}goire and Fourrier, Cl{\'e}mentine and Swift, Craig and Wolf, Thomas and LeCun, Yann and Scialom, Thomas},
  title = {GAIA: A Benchmark for General AI Assistants},
  journal = {arXiv preprint arXiv:2311.12983},
  year = {2023}
}

@article{liu2023agentbench,
  author = {Liu, Xiao and Yu, Hao and Zhang, Hanchen and Xu, Yifan and Lei, Xuanyu and Lai, Hanyu and Gu, Yu and Ding, Hangliang and Men, Kaiwen and Yang, Kejuan and others},
  title = {AgentBench: Evaluating LLMs as Agents},
  journal = {arXiv preprint arXiv:2308.03688},
  year = {2023}
}

@article{liang2023helm,
  author = {Liang, Percy and Bommasani, Rishi and Lee, Tony and Tsipras, Dimitris and Soylu, Dilara and Yasunaga, Michihiro and Zhang, Yian and Narayanan, Deepak and Wu, Yuhuai and Kumar, Ananya and others},
  title = {Holistic Evaluation of Language Models},
  journal = {Transactions on Machine Learning Research},
  year = {2023}
}

@article{brundage2020toward,
  author = {Brundage, Miles and Avin, Shahar and Wang, Jack and Belfield, Haydn and Krueger, Gretchen and Hadfield, Gillian and others},
  title = {Toward Trustworthy AI Development: Mechanisms for Supporting Verifiable Claims},
  journal = {arXiv preprint arXiv:2004.07213},
  year = {2020}
}

@inproceedings{necula1997proof,
  author = {Necula, George C.},
  title = {Proof-Carrying Code},
  booktitle = {ACM SIGPLAN Symposium on Principles of Programming Languages},
  pages = {106--119},
  year = {1997}
}

@misc{rademacher2026standingalgebra,
  author = {Rademacher, Jonathan},
  title = {Standing Algebra {$\Sigma^R$}: A Closure-Theoretic Operator for Constraining Domination and Preserving Autonomy},
  howpublished = {Zenodo working paper, Version 6.5},
  year = {2026},
  month = apr,
  doi = {10.5281/zenodo.19656146},
  url = {https://doi.org/10.5281/zenodo.19656146}
}

\end{document}